\documentclass[letterpaper, 10 pt, conference]{ieeeconf}

\IEEEoverridecommandlockouts

\usepackage{amsmath, amsfonts, algorithm, algpseudocode, graphicx}
\usepackage{cite}
\usepackage{amsthm}
\usepackage{xcolor}
\usepackage{booktabs}
\usepackage{bbold}
\theoremstyle{plain}

\newtheorem{prop}{Proposition}
\usepackage{standalone}
\usepackage{tikz}
\usepackage{subcaption}
\usetikzlibrary{automata, calc, positioning}
\theoremstyle{definition}
\newtheorem{defn}{Definition}
\newtheorem{prob}{Problem}

\tikzstyle{startstop} = [ellipse, draw, align=center, minimum width=2.5cm, minimum height=1cm]
\tikzstyle{process} = [rectangle, draw, align=center, minimum width=3.5cm, minimum height=1cm]
\tikzstyle{decision} = [diamond, draw, aspect=2, align=center, inner sep=2pt]
\tikzstyle{arrow} = [->, thick]
\let\labelindent\relax
\usepackage{enumitem}
\usepackage{stfloats}

\usetikzlibrary{arrows.meta, positioning, shapes.geometric}

\title{\LARGE \bf
Safe, Real-Time Active Model Discrimination and Fault Diagnosis for Nonlinear Systems via Differentiable Reachability
}

\author{Xinpei Ni$^{1}$ \and Melkior Ornik$^{2}$ \and Glen Chou$^{1}$ \and Samuel Coogan$^{1}$
\thanks{$^{1}$X. Ni, G. Chou and S. Coogan are with the Institute of Robotics and Intelligent Machines (IRIM), Georgia Institute of Technology, Atlanta, GA 30332, USA
        {\tt\small \{xni32, chou, sam.coogan\}@gatech.edu}}
\thanks{$^{2}$Melkior Ornik is with the Department of Aerospace Engineering, University of Illinois Urbana-Champaign,
        Urbana, IL 61801, USA
        {\tt\small mornik@illinois.edu}}
}

\begin{document}

\maketitle
\thispagestyle{empty}
\pagestyle{empty}
\vspace{-30em}

\begin{abstract}

We present a safe, real-time algorithm for active fault diagnosis and model discrimination for uncertain continuous-time nonlinear systems with process and measurement disturbances. Given a finite set of candidate models representing nominal and faulty modes, including actuator and sensor faults, we formulate an output-feedback, time-varying policy optimization problem that (i) robustly enforces state-input safety constraints over a finite horizon and (ii) drives the system to produce sampled measurements consistent with at most one model, enabling deterministic diagnosis. To solve this problem in real time, we develop a tractable approximation using interval over-approximations of reachable state and output sets, and encode diagnosability via a differentiable objective that penalizes overlap between the reachable output sets of possible models. The resulting optimization is solved efficiently online with gradient-based methods using JAX and differentiable reachability primitives. We evaluate our method on sensor and actuator fault diagnosis (up to 11 fault modes) in several high-dimensional nonlinear robotic systems, including a simulated quadrotor and fighter-jet model, a hardware differential-drive robot, and quadrupedal navigation. Across these case studies, our approach achieves reliable model discrimination in under 50 ms, outperforming baselines in discrimination success rate and speed while providing formal safety guarantees.

\end{abstract}

\begin{figure*}[!b]\vspace{-10pt}
    \centering
    \begin{subfigure}{0.74\textwidth}
        \centering
    \begin{tikzpicture}
        \node[anchor=south west, inner sep=0] (img)
            at (0,0) {\includegraphics[width=\linewidth, trim=0 1cm 0 0.7cm,clip]{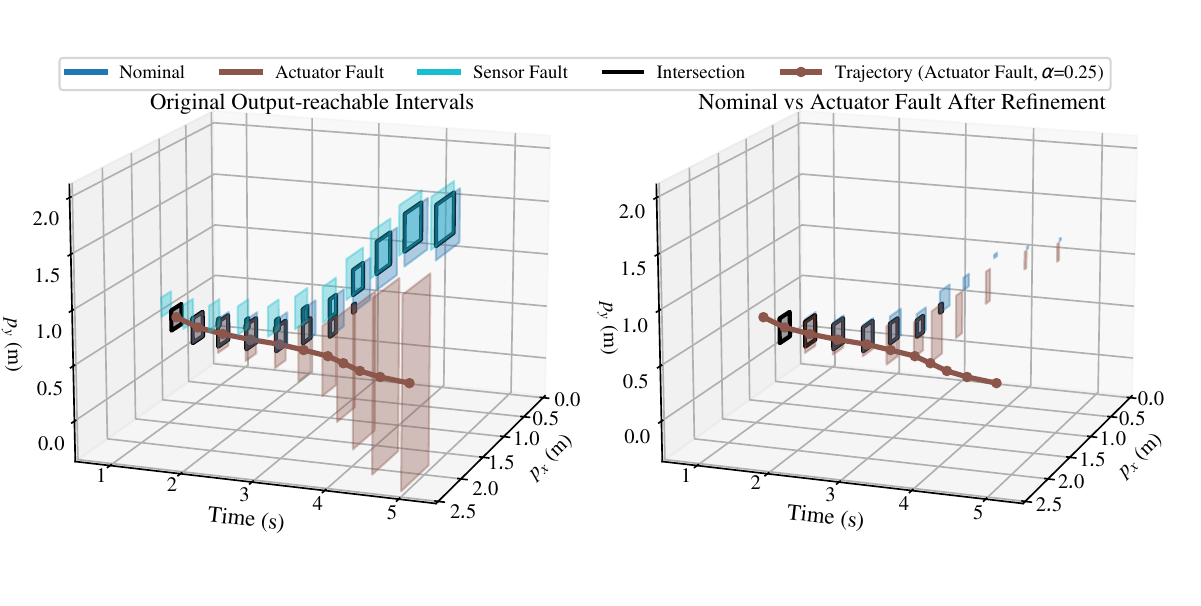}};
        \node[anchor=north west, font=\bfseries]
            at (img.north west) {(A)};
    \end{tikzpicture}

    \label{fig:unicycle_top}
    \end{subfigure}
    \begin{subfigure}{0.25\textwidth}\vspace{-10pt}
    \begin{tikzpicture}
    \node[anchor=south west, inner sep=0] (img)
        at (0,0) {\includegraphics[width=\linewidth]{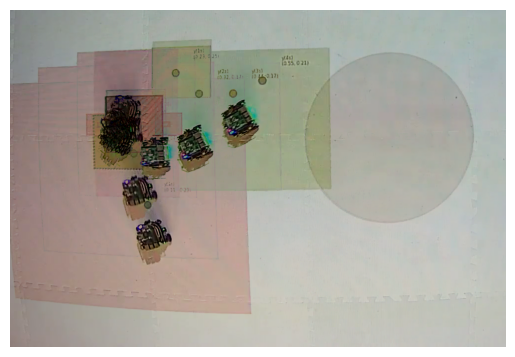}};
    \node[anchor=north west, font=\bfseries]
        at (img.north west) {(B)};
    \end{tikzpicture}
    \begin{tikzpicture}
    \node[anchor=south west, inner sep=0] (img)
        at (0,0) {\includegraphics[width=\linewidth]{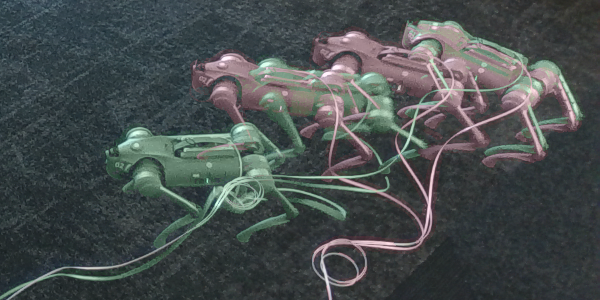}};
    \node[anchor=north west, font=\bfseries]
        at (img.north west) {\textcolor{white}{(C)}};
    \end{tikzpicture}
    \end{subfigure}
    \caption{\textbf{(A)} Illustration of output reachable intervals with and without sensor anticipating refinement. Using output-anticipating refinement allows the intersections of output reachable intervals to shrink faster and enables faster diagnosis. \textbf{(B)} Top view of a hardware experiment in which a separating output-feedback controller is computed by (\ref{opt:active-intersection-refinement}) and applied on a unicycle system. The controller successfully discriminates a sensor fault from an actuator fault in 5 seconds while ensuring that the system's \emph{true position} avoids an obstacle shaded in gray. While the output reachable set for sensor fault (in green) overlaps with the obstacle, this is caused by the sensor fault's output corruption and the robot's \emph{true state} i.e. position remains clear of the obstacle at all times. \textbf{(C)} Angled view of a hardware experiment where a separating controller is computed by \eqref{eq:opt_cost} and applied on a quadruped to diagnose actuator and sensor faults in two separate runs. Frames corresponding to the actuator fault are tinted in red, and frames corresponding to the sensor fault are tinted green. The controller successfully diagnoses the fault in 2 seconds.}
    \label{fig:placeholder}
\end{figure*}
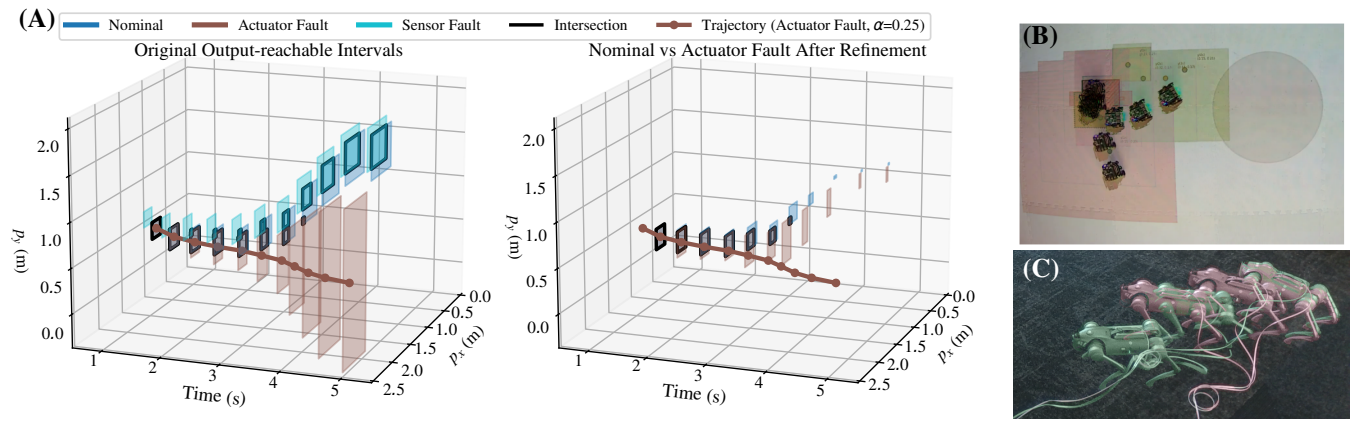
\section{Introduction}

Faults are inevitable in long-duration robotic deployments. For instance, sensors can drift or saturate, actuators gradually lose effectiveness, and unmodeled dynamics can emerge as hardware wears. In safety-critical settings, delayed or incorrect diagnosis can be catastrophic: controllers that trust biased measurements or command degraded actuators may quickly violate safety constraints or destabilize the system. Thus, robots must not only detect that a fault has occurred, but also isolate \emph{which} fault mode is valid quickly enough to enable reconfiguration or fail-safe behavior.

A standard formalization of fault isolation is \emph{model discrimination} \cite{reilly1970statistical}: given a finite set of candidate dynamical models (e.g., a nominal model and multiple fault models), determine which model is consistent with observed behavior. This task is difficult because diagnosis must be performed from noisy, partial measurements under disturbances and uncertainty. Moreover, many distinct candidate models can produce similar measurement signals. For instance, a quadruped with a degraded leg may appear similar to a healthy quadruped with a biased IMU when commanded to walk straight. Control synthesis, however, can simplify diagnosis: for the quadruped, a carefully chosen turning maneuver may generate sensor measurements that disambiguate the failure mode. This motivates \emph{active fault diagnosis} (AFD): rather than passively monitoring outputs, we design control inputs or policies that deliberately excite the system to make the true model identifiable.

However, designing such informative inputs online is challenging. The policy must robustly satisfy state--input safety constraints under \emph{every} candidate fault model and uncertainty realization, while also ensuring that sampled measurements become consistent with \emph{at most one} model so that diagnosis is deterministic. For high-dimensional nonlinear continuous-time systems, exact reachable-set computation is intractable, and most set-based AFD methods are restricted to linear/affine dynamics or require expensive offline computation \cite{harirchi2016model, singh_input_2018}. Conversely, statistical and learning-based approaches can scale but typically lack formal guarantees on safety and correct isolation under worst-case uncertainty.

In this paper, we bridge this gap with a safe, real-time method for active model discrimination and fault diagnosis in uncertain nonlinear systems with process and measurement disturbances.
We formulate AFD as an output-feedback, time-varying policy optimization problem that couples robust safety with a diagnosability requirement over sampled outputs. To make the problem tractable, we compute interval over-approximations of reachable \emph{state} and \emph{output} sets and encode discrimination via a differentiable objective that penalizes overlap between the reachable output sets of the possible models. This yields an optimization amenable to efficient gradient-based solution, enabling fast replanning as candidate models are pruned online. 
Our contributions are:
\begin{itemize}[leftmargin=0.6cm]
    \item An output-feedback policy optimization formulation for safe active fault diagnosis in uncertain nonlinear systems.
    \item Tractable, differentiable over-approximations of reachable output sets and a separation objective for model discrimination.
    \item A novel output-aware algorithm to optimize for a separating controller, with improved performance resulting from intersection refinements.
    \item A fast GPU-parallelized implementation enabling online replanning and model pruning.
    \item Experiments on high-dimensional robotic case studies, including sensor/actuator faults (up to 11 modes) on a simulated fighter-jet model, a hardware differential-drive robot, and quadrupedal navigation \textcolor{black}{in under 50 ms}.
\end{itemize}

\section{Related Work}

\subsection{Statistical, Data-Driven, and Discrete Methods}\label{sec:related_statistical}
Fault detection and model discrimination are ubiquitous tasks across science and engineering, and have therefore been studied extensively. Early work on model discrimination emerged in the chemistry community in the 1970s, where Bayesian methods were used to probabilistically distinguish among candidate mathematical models of chemical processes \cite{reilly1970statistical}. In reliability engineering, fault trees \cite{de2020obtaining} are commonly used for fault detection, but they primarily adopt a discrete-event perspective; in contrast, we consider fault diagnosis for systems with continuous state spaces. Related probabilistic approaches to fault detection and diagnosis have also been explored in the control and robotics communities~\cite{xu_joint_2025, patan2026robust, khalastchi2018sensor, golombek2011online, streif2014optimal}. However, these methods are fundamentally statistical and typically do not provide formal guarantees on either fault identifiability or safety. Other approaches leverage deep learning by training classifiers to detect faults~\cite{garg2023model, park2021data, wang2020active}. While such methods can be highly effective within the training distribution, they are largely heuristic and generally lack formal guarantees. In contrast, our method ensures safety while performing active fault detection in real time with formal correctness guarantees on detection.

\subsection{Verification and Set-Based Methods}
\looseness-1Set-based approaches have been proposed for both passive~\cite{rosa2013fault} and active~\cite{scott2014input} fault diagnosis. These methods compute (over-approximations of) reachable sets under different fault hypotheses and eliminate hypotheses that are inconsistent with the measured outputs. By reasoning over sets, they can rule out fault modes deterministically, rather than probabilistically as in Bayesian formulations (Sec. \ref{sec:related_statistical}).

Due to the computational challenges of reachability analysis, most existing set-based techniques are limited to linear or affine dynamics. For example, \cite{scott2014input} uses a zonotope-based observer and a mixed-integer program to design active inputs for linear systems, and~\cite{harirchi2018guaranteed} establishes guaranteed $T$-detectability for distinguishing affine models. Moreover, \cite{marseglia2017active} considers affine systems under output measurements but requires costly multi-parametric optimization, which precludes real-time use. Several works address nonlinear systems but remain passive: \cite{mu2024set, rego2025zeta} study set-based fault detection for nonlinear dynamics, while \cite{harirchi2016model} provides guaranteed passive discrimination for polynomial dynamics via moment relaxations, at the cost of solving expensive optimization problems that are not amenable to real-time operation. Other nonlinear safety-oriented methods, such as barrier-function approaches \cite{ballotta2025fault, zhang2025safe}, can ensure safety under faults but do not directly enable active discrimination; in particular, they typically only passively detect faults or can only detect faults that drive the system toward constraint violation. 

Only a small number of papers study set-based \emph{active} model discrimination for \emph{nonlinear} systems. The approach in \cite{singh_input_2018} designs informative inputs for nonlinear dynamics but relies on integer programming and does not scale well with dimension. The method of \cite{rudolph2016set} addresses polynomial and rational dynamics via a computationally expensive bilevel program, and is therefore restricted to small systems and to discriminating between two candidate models. A central difficulty in set-based model discrimination methods is computing reachable \emph{output} sets under output-feedback control, i.e., bounding the set of measurement outputs that can arise under uncertainty, which is needed to certify fault diagnosis and safety under faulty conditions. While there has been progress on reachable-set computation for nonlinear systems under output feedback~\cite{dean2021guaranteeing, chou2023synthesizing, chou2022safe, leeman2026vision}, these methods remain too slow for real-time deployment or are not differentiable. In contrast to the methods above, in this work, we build on recent advances in efficient reachable \emph{state}-set computation~\cite{harapanahalli2024immrax} to enable real-time set-based active model discrimination and fault detection for nonlinear systems via fast and differentiable GPU-accelerated reachability analysis.

\section{Problem Formulation}

\noindent We consider a set of $M$ uncertain nonlinear continuous-time dynamical systems $f_m$ and measurement output maps $h_m$
\begin{subequations}\label{eq:system}
    \begin{align}
        \dot{x}_m(t) &= f_m(x_m(t), u(t), w(t), \theta_a),  \label{eq:system_1} \\
        y_m(t) &= h_m(x_m(t), u(t), v(t), \theta_o), \label{eq:system_2}
    \end{align} 
\end{subequations}
\looseness-1indexed by $m \in \{0,\ldots,M-1\}$, with state $x_m\in\mathcal{X}$, control $u\in \mathcal{U}$, process noise $w \in \mathcal{W}$, measurement output $y\in\mathcal{Y}$, and measurement noise $v \in \mathcal{V}$. Here, $f_m: \mathcal{X} \times \mathcal{U} \times \mathcal{W} \times \Theta_a \rightarrow \mathcal{X}$ and $h_m: \mathcal{X} \times \mathcal{U} \times \mathcal{V} \times \Theta_o \rightarrow \mathcal{Y}$ are differentiable functions.
We denote actuator parameters as $\theta_a \in \Theta_a$ and output parameters as $\theta_o \in \Theta_o$. 
We let $(f_0, h_0)$, i.e., $m=0$, denote the \textit{nominal} model which defines the state evolution and measurements of the healthy system, with parameters $\theta_a^* \in \Theta_a$ and $\theta_o^* \in \Theta_o$. The remaining models $(f_m, h_m)$, for $m \in \{1,\ldots,M-1\}$, denote various fault modes. We denote the set of systems as $\mathcal{M}:=\{(f_m,h_m)\}_{m=0}^{M-1}$ and consider the following fault modes:
\begin{itemize}
    \item \textit{Actuator fault}: where there is a deviation in the dynamics function $f_m(x,u, w, \theta_a)$, i.e., $f_m \ne f_0$ and $h_m = h_0$,
    \item \textit{Sensor fault}: where there is a deviation in the output function $h_m(x, u, v, \theta_o)$, i.e., $f_m = f_0$ and $h_m \ne h_0$,
    \item \textit{Simultaneous fault}: where there are deviations in both the dynamics and output, i.e., $f_m \ne f_0$ and $h_m \ne h_0$.
\end{itemize}

\begin{defn}
\textit{Sampled measurement set.\quad} We assume that over a time interval $[0, T]$, measurements are collected from the system at a discrete set of $N_t$ pre-determined time instants $\hat{\mathcal{T}}:=\{\hat t_1,\ldots,\hat t_{N_t}\} \subset [0,T]$. Assuming the system is operating under the true model $m^* \in \{0,\ldots,M-1\}$, this yields the \textit{sampled measurement set} $\hat{Y} :=\{\hat y_i := y_{m^*}(\hat{t}_i)\}_{i=1}^{N_t}$. Let $\hat{\mathcal{Y}}$ denote the set of all possible $\hat Y$ that can be realized under different disturbances $w(\cdot)$ and $v(\cdot)$.
\end{defn}

\begin{defn}
    \textit{Consistency set.\quad} \looseness-1We define $\mathcal{I}(\hat{{Y}})$ as the set of models consistent with the sampled measurement set $\hat{Y}$ as
\begin{equation}
\begin{split}
    \mathcal{I}(\hat{{Y}}) := \{& m \in \{0,\ldots,M-1\} \mid \exists x(0) \in \mathcal{X}_0, \\ & \exists w(\cdot)\in \mathcal{W},\ \exists v(\cdot)\in \mathcal{V}, \\ & \textrm{s.t. } y_m(\hat t_i) = \hat y_i,\quad \forall i \in \{1,\ldots, N_t\} \}.
\end{split}
\end{equation}

\end{defn}

\noindent \looseness-1We consider these \textit{guaranteed} fault diagnosis problems:
\begin{prob}\label{prob:passive-fd}
    \textit{Passive Model Discrimination.\quad}
    Given $\mathcal{M}$, $\mathcal{W}$, $\mathcal{V}$, a set of possible initial states $\mathcal{X}_0$, and a control input sequence $u{(\cdot)}$, we wish to identify which model $m\in\mathcal{M}$ is consistent with the system's recorded outputs, or none exists.

\end{prob}

This problem can be solved by Alg. \ref{alg:GENERIC_FD}.
\begin{algorithm}
\caption{Fault Diagnosis}\label{alg:GENERIC_FD}
\begin{algorithmic}[1]

\State \textbf{Input:} Initial model set $\mathcal{M}$, observation intervals $\mathcal{R}_{y,i}$, control sequence $u(\cdot)=\{u_0, u_1, \dots u_N\}$

\State Apply $u_0$ to system

\For{$k = 1, \dots, N$}
    \State Measure output $y(t_k)$

    \For{$m \in \mathcal{M}$}
        \If{$y(t_k) \notin \mathcal{R}^{m}_y(t_k, u_k)$}
            \State $\mathcal{M} \leftarrow \mathcal{M} \setminus \{m\}$ \Comment{Prune inconsistent model}
        \EndIf
    \EndFor

    \If{$|\mathcal{M}| \le 1$}
        \State \textbf{return} $\mathcal{M}$ \Comment{Diagnosis complete}
    \EndIf

    \State Apply $u_k$ to system

\EndFor
\end{algorithmic}
\end{algorithm}

An ideal $u(\cdot)$ for solving Prob. \ref{prob:passive-fd} is one that can reduce $\mathcal{M}$ to a singleton, and computing such an optimal control sequence $u$ that solves Prob. \ref{prob:passive-fd} requires solving the active fault diagnosis problem:
\begin{prob}\label{prob}
    \textit{Active Fault Diagnosis.\quad }
Given $\mathcal{M}$, $\mathcal{W}$, $\mathcal{V}$, and a set of possible initial states $\mathcal{X}_0$,
we wish to design a piecewise continuous, time-varying control policy $\pi: [0, T]\times \mathcal{Y} \rightarrow \mathcal{U}$ which ensures that 1) any possible resulting sampled measurement set $\hat{\mathcal{Y}}$ is consistent with at most one model in $\mathcal{M}$ and 2) a prescribed set of state-control constraints $\mathcal{S}\subseteq \mathcal{X} \times \mathcal{U}$ encoding, e.g., collision avoidance and actuator limits, is robustly satisfied for all $t \in [0,T]$, \emph{i.e.}, for all $w \in \mathcal{W}$, $v \in \mathcal{V}$, and all models in $\mathcal{M}$.
We formalize this problem as the following feasibility problem \eqref{eq:afd_opt}:
\begin{subequations}\label{eq:afd_opt}
\begin{align}
\text{find}\quad &\pi \label{eq:afd_obj}\\
\text{s.t.}\quad 
&\dot x_m(t)= f_{m}(x_m(t),\pi(t, y_m(t)),w(t), \theta_a^m),\nonumber\\
& \qquad \forall t\in[0,T],\quad   \forall m \in \{0,\ldots,M-1\},\label{eq:afd_dyn}\\
&y_m(t)= h_{m}(x_m(t),,\pi(t, y_m(t)),v(t),\theta_o^m), \quad \nonumber\\
& \qquad \forall t\in[0,T],\quad   \forall m \in \{0,\ldots,M-1\},\label{eq:afd_out}\\
&x_m(0)\in\mathcal{X}_0,\quad \forall m\in\{0,\ldots,M-1\}, \label{eq:afd_init}\\
&(x_m(t),\pi(t, y_m(t)))\in\mathcal{S}, \quad \forall t\in[0,T], \nonumber\\
& \qquad \forall w(\cdot)\in\mathcal{W},\quad \forall v(\cdot)\in\mathcal{V}, \nonumber\\
& \qquad \forall m \in \{0, \ldots, M-1\}, \label{eq:afd_state_input}\\
& |\mathcal{I}(\hat{Y})| \leq 1, \quad \forall \hat Y \in \hat{\mathcal{Y}}.\label{eq:afd_consistency}
\end{align}
\end{subequations}

The algorithm for solving \eqref{eq:afd_opt} must be:
\begin{itemize}
    \item efficient enough to compute in real time, 
    \item scalable to high-dimensional nonlinear systems, and
    \item sound; i.e., it identifies that the system is operating under model $m$ if and only if model $m$ is truly active.
\end{itemize}
\end{prob}

In practice, exactly solving \eqref{eq:afd_opt} is intractable, since computing exact state and output reachable sets for nonlinear systems is itself intractable, and these sets are required to enforce \eqref{eq:afd_state_input} and \eqref{eq:afd_consistency}. In what follows, we provide preliminary background in Sec. \ref{sec:preliminaries} and then describe our approximate solution method for \eqref{eq:afd_opt} in Sec. \ref{sec:method}.

\section{Preliminaries}\label{sec:preliminaries}
\subsection{State and Output Reachable Sets}
Given $(f_m, h_m)$ for some model $m \in \{0, \ldots, M-1\}$ as defined in \eqref{eq:system} and an initial state set $\mathcal{X}_{0}$, we define its \textit{state reachable set} at time $T$ as:
\begin{equation} \label{eq:state_reach_exact}
\begin{split}
    \hspace{-8pt}\mathcal{X}_T^m := \{&x(T) \in \mathcal{X} \mid\; x(0) \in \mathcal{X}_{0},\ \dot x_m=f_m(x_m,u,w),\\
    &  u(\tau) \in \mathcal{U}, w(\tau) \in \mathcal{W}, \,\forall \tau \in [0, T)\}.
\end{split}
\end{equation}
Similarly, the \textit{output reachable set} is defined as the image of the state reachable set under the observation map $h_m$:
\begin{equation} \label{eq:output_reach_exact}
    \mathcal{Y}_T^m := \{h_m(x, u, v) \mid x \in \mathcal{X}_T^m, u \in \pi(t,\mathcal{X}_t), v\in \mathcal{V}\}.
\end{equation}

Because computing (\ref{eq:state_reach_exact}) and (\ref{eq:output_reach_exact}) exactly is difficult, we exploit mixed monotonicity \cite{coogan2020mixed} to efficiently compute their overapproximations. 
For the above system, we leverage interval analysis \cite{jaulin1993set, didrit2001applied} to obtain \emph{inclusion functions} $\mathsf{F}_m=[\underline{F_m}, \overline{F_m}], \mathsf{H}_m=[\underline{H_m}, \overline{H_m}]$ over a given \emph{interval} $[x]=[\underline{x}, \overline{x}], [v]=[\underline{v}, \overline{v}], [w]=[\underline{w}, \overline{w}]$ such that $\underline{F_m}([x], u, [v]) \leq  f_m(x) \leq  \overline{F_m}([x],u,[v]])$ and $\underline{H_m}([x], u, [w]) \leq  h_m(x) \leq  \overline{H_m}([x],u,[w]])$ for all $x \in [x], v\in[v], w\in[w]$. This set of inclusion functions induces an \emph{embedding system} \cite{abate2020computing} $(\mathsf{F}_m, \mathsf{H}_m)$ for the system $(f_m, h_m)$.
\subsection{Approximating Reachable Sets with \texttt{immrax}} \label{sec:prelim-immrax}
To efficiently compute such embedding systems, we turn to \texttt{immrax} \cite{harapanahalli2024immrax}, a parallelizable and differentiable JAX-based Python toolbox that facilitates interval analysis and mixed monotone reachability through composable function transforms. \texttt{immrax} leverages core features of JAX, a Python package for high-performance computation, such as just-in-time (JIT) compilation, GPU acceleration, and automatic differentiation, to efficiently compute sound reachable set over-approximations. 
This process is highly scalable, as it allows for rapid reachable set refinement via parallelized state-space partitioning on GPUs. In prior work \cite{harapanahalli2024immrax}, \texttt{immrax} has been used to compute state reachable sets; we extend it to compute output reachable sets for use here.

\begin{defn} \textit{Interval Approximation of Output Reachable Set.}
\label{def:RF}
    Given a system model like (\ref{eq:system}), a time horizon $T$, inclusion functions $\mathsf{F}_m$ for $f_m$ and $\mathsf{H}_m$ for $h_m$, and an interval overapproximation of state reachable set at time $T$, $\mathcal{R}_x(T)$, a corresponding \emph{interval approximation of its output reachable set at time $T$} is defined as $\mathcal{R}_y(T, \mathcal{R}_x, u, [v]):=[\underline{H}_m(\mathcal{R}_x(T), u, [v]), \overline{H}_m(\mathcal{R}_x(T), u, [v])]$ when it is subject to measurement noise $v \in [v]$ and input $u$ over all $t\in[0, T]$.
\end{defn}

\begin{prop}
$\mathcal{R}_y$ as given in Definition \ref{def:RF} overapproximates the true output-reachable set $\mathcal{Y}$; i.e., $\mathcal{Y} \subseteq \mathcal{R}_y$.
\end{prop}
\noindent \textbf{Proof.} For any $y \in \mathcal{Y}$, there exists $x \in \mathcal{X}, v\in\mathcal{V}$ such that $y=h_m(x,u, v)$ given control input $u$. Because $\mathcal{R}_x$ overapproximates $\mathcal{X}$, $\mathcal{X} \subseteq \mathcal{R}_x$. As a property of inclusion functions, we have $\underline{H}_m(\mathcal{R}_x, u, \mathcal{V}) \leq h_m(x,u,v)\leq\overline{H}_m(\mathcal{R}_x,u,\mathcal{V})$. Because $\mathcal{R}_y=[\underline{H}_m(\mathcal{R}_x, u, \mathcal{V}), \overline{H}_m(\mathcal{R}_x, u, \mathcal{V})]$, it follows that $y=h_m(x,u,v)\in\mathcal{R}_y$, and hence $\mathcal{Y}\subseteq\mathcal{R}_y. \hfill \Box$

Relative to \eqref{eq:afd_opt}, robust safety constraint satisfaction is enforced via the reachable set over-approximations $\mathcal{R}_x$, and the consistency constraint \eqref{eq:afd_consistency} is relaxed to a cost penalty.

\section{Method}\label{sec:method}

We now describe our method, which conservatively solves \eqref{eq:afd_opt}. First, we show how to encode the model-discrimination constraint \eqref{eq:afd_consistency} as a differentiable objective (Sec. \ref{sec:method_objective}). Using a tractable over-approximation of the state and output reachable sets (Sec. \ref{sec:prelim-immrax}) to obtain an optimization problem that can be implemented efficiently in JAX (Sec. \ref{sec:method_implementation}), we formulate an uninformed but functional algorithm for solving separating controllers online (Sec. \ref{sec:naive_algorithm}). We then explore a output-anticipating refinement scheme that reduces conservatism and hence discriminates models faster in Sec. \ref{sec:new_algorithm}, and discuss its prerequisites for use in Sec. \ref{sec:air-prerequisites}.

\subsection{Discrimination Objective via Output Set Overlap}\label{sec:method_objective}
To efficiently solve \eqref{eq:afd_opt}, one of the core difficulties is in encoding \eqref{eq:afd_consistency}. 
To encode this, we note that any feasible solution of \eqref{eq:afd_consistency} guarantees that any realization of $\pi$ under disturbance is a \textit{separating input sequence} $u_\textrm{sep}:=\{\hat u_i :=\pi(t_i, y(\hat t_i))\}_{i=1}^{N_t}$ for model set $\mathcal{M}$, i.e., executing $u_\textrm{sep}$ ensures that for all $i \in\{0,\ldots,M-1\}$ and $j \in \{0,\ldots,M-1\}\setminus i$, $\mathcal{Y}_i \cap \mathcal{Y}_j = \emptyset$.

For \eqref{eq:system}, and $N_t$ outputs taken at times $t\in \hat{\mathcal{T}} \subset \mathcal{T}:=[0, T]$, denoting the set of all model pairs $q:=\{j, k\}$ as $\mathcal{Q}:=\{\{j, k\}\mid j,k\in\mathcal{M}, j\neq k\}$, we would like to find a control policy $\pi(t, y(t))$ such that for all pairs $q\in\mathcal{Q}$, there exists a time $t^* \in \hat{\mathcal{T}}$ such that the intersection of the fault cases' corresponding output reachable sets at $t$ under controller $\pi(\cdot)$ is empty; i.e., 
\begin{align}
\nonumber &\forall q=\{j,k\}\in\mathcal{Q},\ \exists t_i\in\hat{\mathcal{T}}:\\
&\qquad\mathcal{Y}_j(t_i, \pi(t, y)) \cap \mathcal{Y}_k(t_i, \pi(t, y)) = \emptyset, \label{constraint:no_overlap}
\end{align}
while requiring that $u$ satisfies input constraints and that the system satisfies collision avoidance constraints. Since there are cases where there does not exist a separating controller, we include \eqref{constraint:no_overlap} as a term to be minimized in the objective $J$:

\begin{equation}
\begin{aligned}
    J(\pi)&:=\sum_{q\in \mathcal{Q}} J_q(\pi) \\ 
    J_q(\pi)&:= \min_{t \in \hat{\mathcal{T}}} 
    \textsf{Vol}\left( \bigcap_{m\in q} \mathcal{Y}_m(t,\pi(t,y)) \right).\label{eq:product_cost}
\end{aligned}
\end{equation}
We justify the use of \eqref{eq:product_cost} through the following theorem. 
\begin{prop}If a system admits a separating controller $\pi$ for the model set $\mathcal{M}$, then $\pi$ is a global minimizer of cost function defined in \eqref{eq:product_cost}. \end{prop}
\noindent\textbf{Proof}. By definition, the separating control $u$ satisfies that for each pair of models $ q=\{j,k\} \in \mathcal{Q, } \text{ there exists a time } t \text{ such that } \mathcal{Y}_j \cap \mathcal{Y}_k = \emptyset$, which sets $J_q(\pi)=0$. 
As a result, the summation (\ref{eq:product_cost}) is zero.
Since this summation is non-negative, zero is its global minimum, and thus this $\pi$ minimizes $J(\pi). \hfill \Box$

This theorem shows that any separating input yields zero objective value according to \eqref{eq:product_cost}. In particular, this is achieved when all pairs of possible models $i,j$, where $j \ne i$, have disjoint output reachable sets at some measured time instant $t \in \hat{\mathcal{T}}$. 
This implies that over the full horizon $[0,T]$, the resulting output sequence is contained only in the output reachable set of a single model.
While the reverse is not true for general systems, for any solution with a separating input $u$, the solution to (\ref{eq:product_cost}) will ensure that for any pair $q=\{j,k\}\in \mathcal{Q}\text{, there exists a time } t_i \in \hat{\mathcal{T}}, \text{when } \textsf{Vol}(\mathcal{Y}_j(t_i, \pi(t, y)) \cap \mathcal{Y}_k(t_i, \pi(t, y)))=0$, which means the intersection of these models' output reachable sets will be of measure zero. In practice, this is often sufficient for robust discrimination.

It shall be noted that the volumetric overlap $\textsf{Vol}$ used in (\ref{eq:opt_cost}) can be substituted by many other metrics that can measure separation between two compact sets, such as Hausdorff distance \cite{jaulin1993set}, that are not used in this paper. We leave evaluation of these metrics as future work.

\subsection{Differentiable, Parallel Volumetric Overlap in \texttt{immrax}}
\label{sec:method_implementation}
\textcolor{red}{}
To solve \eqref{eq:opt_cost} efficiently, we utilize the fact that \texttt{immrax} provides differentiable inclusion functions.

While it is difficult to compute volumes of exact reachable sets and their intersections, their interval overapproximations are much easier to work with. Any two intervals $[a]=[\underline{a}, \overline{a}]$ and $[b]=[\underline{b}, \overline{b}]$ intersect if and only if $\max (\underline{a}, \underline{b}) \leq \min(\overline{a}, \overline{b})$ componentwise. The volume of their intersection is $\mathsf{Vol}([a]\cap [b])=\Pi_i[\min(\overline{a}, \overline{b}) - \max (\underline{a}, \underline{b})]_i$. Then, given any two \textit{interval overapproximations} of reachable sets, we can substitute them into the previous formula to calculate their intersection in a quick, differentiable, and parallelizable manner using an efficient JAX implementation. 
See \cite{harapanahalli2024immrax} for discussion on conservativeness of these approximations.

This allows us to compute the gradient of the reachable set bounds with respect to the control parameters $\frac{\partial \mathcal{R}_y}{\partial u}$ using JAX's automatic differentiation. We solve \eqref{eq:opt_cost} to local optimality using a gradient-based optimizer (e.g., Adam), where the state trajectories for all $n_m$ models are simulated in parallel on a GPU.

\subsection{Uninformed, Single-Pass Reachability Optimization for Model Discrimination} \label{sec:naive_algorithm}
Directly implementing the components established so far in JAX, we arrive at an optimization problem that is effective in practice. In particular, it integrates interval reachable-set propagation, the overlap-volume discrimination objective, and robust safety constraints within a single formulation.

We now present an optimization problem for active fault diagnosis. We name this formulation the \emph{uninformed algorithm} because it does not explicitly leverage the fact that sensor measurements will be taken along the execution of the control sequence when computing a controller.
\begin{subequations} \label{eq:opt_cost}
\begin{align}
    \min_{u} \quad&  J(u) := \sum_{q\in \mathcal{Q}} J_q(u) \label{eq:cost_func} \\
    \text{s.t.} \quad&  J_q(u_k) = \min_{t_k\in \hat{\mathcal{T}}} {\mathsf{Vol}} \left[ \bigcap_{m\in q} [y]^m(t_k)  \right]\\
    &[x]^m(t_{k+1}) = \mathcal{R}_x^m(t_{k+1},[x]^m(t_{k}),u_k, [w]) \label{eq:state-propagate} \\
    &[y]^m(t_{k}) = \mathcal{R}_y^m(t_k,[x]^m(t_{k}),u_k, [v]) \nonumber\\
    & \hspace{35pt} = \mathsf{H}_m([x]^m(t_k),u_k,[v])  \label{eq:output-propagate} \\
     &\mathcal{R}_{x}^m(t) \subseteq \mathcal{S}, \quad \forall t \in [0,T],\ \forall m \in \mathcal{M} \label{eq:const_safety} \\
     &u(t) \in \mathcal{U}, \quad \forall t \in \mathcal{T}
\end{align}
\end{subequations}
\begin{figure}[h]
    \centering
    \begin{tikzpicture}[
    box/.style={
        draw,
        minimum width=2cm,
        minimum height=2cm,
        fill opacity=0.4,
        text=black,
        text opacity=1
    },
    af_ori/.style={box, fill=red!60},
    sf_ori/.style={box, fill=green!60},
    af_ref/.style={
        draw,
        minimum width=0.8cm,
        minimum height=1.2cm,
        fill=red!60
    },
    sf_ref/.style={
        draw,
        minimum width=1.2cm,
        minimum height=0.8cm,
        fill=green!60
    },
    arrow/.style={->, thick}
]

% =========================================================
% Left overlapping boxes
% =========================================================

\node[af_ori] (R1) at (-0.4,0.6) {};
\node[sf_ori] (G1) at (0,-0.2) {};

% Labels moved away from intersection
\node at (-.1,1.2) {\small $[y]_j/[y]_k$};
\node at (-0.13,-0.8) {\small $[y]_k/[y]_j$};
\node at (-0.2,-0.1) {\small $[y]_j\cap[y]_k$};

% Compute intersection
\path 
    let \p1 = (R1.south west),
        \p2 = (G1.south west),
        \p3 = (R1.north east),
        \p4 = (G1.north east)
    in
    coordinate (I-sw) at ({max(\x1,\x2)}, {max(\y1,\y2)})
    coordinate (I-ne) at ({min(\x3,\x4)}, {min(\y3,\y4)});

% Draw intersection
\draw[thick] (I-sw) rectangle (I-ne)
    node[pos=0.5]{};% {\small $[y]_j\cap[y]_k$};

% =========================================================
% Region labels A B C
% =========================================================

% A = upper-left non-overlap
\node at (-1.0,1.1) {\large A};

% B = intersection
\node at ($(I-sw)!0.8!(I-ne)$) {\large B};

% C = lower-right non-overlap
\node at (0.7,-.9) {\large C};

% =========================================================
% Right refinement boxes
% =========================================================

\node[af_ref] (R2) at (5.5,.7) {$[x]^q_j$};
\node[sf_ref] (G2) at (5.5,-.7) {$[x]^q_k$};

% =========================================================
% Direct arrows from intersection
% =========================================================

\draw[arrow]
    (I-ne)
    to[out=20, in=180]
    node[midway, above] {\small $\mathsf{H}_j^{-1}$}
    (R2.west);

\draw[arrow]
    ($(I-ne |- I-sw)$)
    to[out=-20, in=180]
    node[midway, below] {\small $\mathsf{H}_k^{-1}$}
    (G2.west);

% =========================================================
% Group labels
% =========================================================

\node[below=1.5cm of $(R1)!0.5!(G1)$]
    {\small Compute output interval};

\node[below=1.3cm of $(R2)!0.5!(G2)$]
    {\small Refine state interval};

\end{tikzpicture}
    \caption{Illustration of output-anticipating refinement process. Output trajectories that land in only one of the two models' output reachable interval such as Regions A and C will be instantly discriminated at this output, and only those landing in the intersection (Region B) require further separation. Region B, which can be expressed as $[y]_q=[y]_j\cap [y]_k$, can then be transformed into refined state intervals $[x]_j^q=\mathsf{H}_j^{-1}([y]_q)\cap [x]_j$ and $[x]_k^q=\mathsf{H}_k^{-1}([y]_q)\cap [x]_k$, each model $j$ having a possibly different $[x]_j^q$ for each pair $q$ containing it. The output-anticipating refinement scheme reduces conservatism by conditioning future computation on this fact.}
    \label{fig:ref_ill}
\vspace{-2em}
\end{figure}
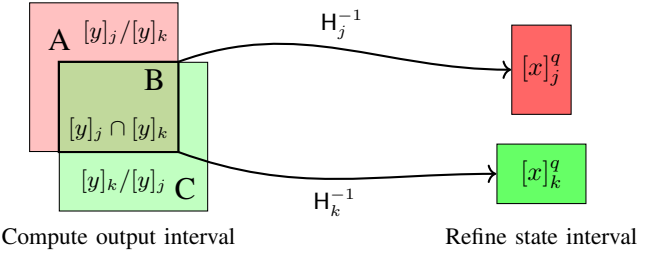
\subsection{Output-Anticipating Refinement for Tighter Model Discrimination}
\label{sec:new_algorithm}
Starting from \eqref{eq:opt_cost}, we refine the state set used for forward propagation by focusing on ambiguity-relevant states. Noticing that finding a controller tsatisfying (\ref{constraint:no_overlap}) is equivalent to separating each pair at \emph{any} time when an output is measured, we argue that when computing a control sequence,  only states that will lead to unseparated outputs require further separation. 
Specifically, at time $t_k$, we keep only states whose outputs lie in pairwise intersections of reachable output sets. 
Assuming a closed-form expression of preimage of output map $\mathsf{H}: \mathcal{X}\rightarrow\mathcal{Y}$ exists, we map this set back to a state interval,
which is then used to initialize subsequent reachable-set propagation. Fig. \ref{fig:ref_ill} provides a visual explanation of this process. 
To further reduce conservatism, we maintain a separate interval for each model pair $q\in\mathcal{Q}$. This pairwise construction directly targets the separation condition in \eqref{eq:afd_consistency}.

The resulting multistep optimization problem is: 
for any time step $k \in \{0, 1, ..., K\}$, where $K$ denotes the horizon or the number of control steps we wish to solve for,
the optimal control $u_k$ is found as:
\vspace{-2em}

\begin{subequations}\label{opt:active-intersection-refinement}
\begin{align}
\min_{u_k \in \mathcal{U}} \quad & J(u_k) = \sum_{q\in \mathcal{Q}(t_{k}) }  J_q(u_k)\\
\text{s.t.} \quad & J_q(u_k) = \min_{t_k\in \hat{\mathcal{T}}} {\mathsf{Vol}} \left[ \bigcap_{m\in q} [y]^m_q(t_k) \right]\\
    & [x]^m(t_{k+1}) = (\ref{eq:state-propagate}), [y]^m(t_k) = (\ref{eq:output-propagate})\\
    & \tilde{[y]}^m_q(t_k) = \bigcap_{m\in q}[y]_q^m(t_k) \label{eq:refined_output_ivl} \\
    & \tilde{[x]}^m_q(t_{k}) = \mathsf{H}_m^{-1}(\tilde{[y]}_m^q(t_k))\cap {[x]}_m^q(t_k) \label{eq:refined_state_ivl}\\
    & [x]_q^m(t_{k+1})=\mathcal{R}_x^m(t_{k+1}|\tilde{[x]}_q^m(t_k), u_k, [w])\label{eq:propagate_refinement}\\
    & [x]^m(\tau) \subseteq \mathcal{S}, \quad \forall \tau \in [t_k, t_{k+1}],\ \forall m \in \mathcal{M} \label{eq:safe-air}\\
    & u(\tau) \in \mathcal{U}, \quad \forall \tau \in [t_k, t_{k+1}] \\
    & \mathcal{Q}(t_{k+1}) = \{q \in \mathcal{Q}(t_{k}) \mid J_q(u_{k}^*)>0\}
\end{align} 
\end{subequations}

\noindent where $[y]^m_q$ is the output reachable set for model $m$ under pair $q$, and $\mathsf{H}^{-1}$ is the preimage of model $m$'s output map. 
This optimization differs from (\ref{eq:opt_cost}) only in (\ref{eq:refined_state_ivl}), (\ref{eq:refined_output_ivl}), and (\ref{eq:propagate_refinement}), which describe our output-anticipating refinement.

\subsection{Requirements for Using Output-Anticipating Refinement} \label{sec:air-prerequisites}

The process of computing a preimage of $\mathsf{H}$ is NP-hard for general functions $h$\cite{didrit2001applied}. To compute $\mathsf{H}^{-1}([y](t_k))$, one may employ subpaving refinement algorithms such as Sivia \cite{jaulin1993set}, which takes an initial guess $\check{\mathsf{X}}, \hat{\mathsf{X}}$ such that $\check{\mathsf{X}} \subset \mathsf{H}^{-1}([y](t_k)) \subset \hat{\mathsf{X}}$ and recursively refines $\check{\mathsf{X}}, \hat{\mathsf{X}}$ to return a tighter estimate of $\mathsf{H}^{-1}([y](t_k))$\cite{jaulin1993set}. For a general $h$, the time complexity of finding  $\mathsf{H}^{-1}([y](t_k))$ in such a fashion is high, thus defeating Algorithm (\ref{opt:active-intersection-refinement})'s purpose of finding a separating controller in real-time speed. Devising fast set-inversion algorithms is an open research topic in interval analysis and is beyond the scope of this paper. Hence, for the sake of this paper, (\ref{opt:active-intersection-refinement}) is only used if there exists a closed form expression for an interval overapproximation $\hat{\mathsf{X}}(t_k) \supseteq \mathsf{H}^{-1}([y](t_k))$. Additionally, if it can be established that the estimate $\hat{\mathsf{X}}(t_k)$ guarantees that $[x](t_k) \subseteq \mathsf{H}^{-1}([y](t_k)) \text{ at all times } t_k$, (\ref{eq:opt_cost}) should be used. Fortunately, many robotic sensors can be modeled by affine output maps of the form $y=h(x,u,v)=Cx+Du+v$. In this case, the relevant preimage can be written as $x=h^{-1}(y,u,v)=C^\dagger(y-Du-v)$, where $C^\dagger$ is left-inverse of $C$, and can be efficiently computed or overapproximated.
\section{Results}
We present hardware and simulation experiments of our proposed active model discrimination scheme on three different systems. The goals for these experiments are:
\begin{itemize}
    \item \textbf{Unicycle} shows (\ref{opt:active-intersection-refinement}) can run in milliseconds on a physical nonlinear system with a small set of models.
    \item \textbf{ADMIRE} demonstrates (\ref{opt:active-intersection-refinement}) can be applied on a high-dimensional nonlinear system with many candidate models and solves at real-time rates.  
    \item \textbf{Go2} shows (\ref{eq:opt_cost}) is a valid choice for sensor fault models involing unobserved states while enforcing safety.
\end{itemize}
Notice that both (\ref{eq:opt_cost}) and (\ref{opt:active-intersection-refinement}) are able to handle input and state constraints. We only showcased a safety constraint on Go2 because we believe maintaining safety under a fault that causes non-measureability yields the strongest result.

\noindent\textbf{Baselines:} Very few works are capable of performing active model discrimination on nonlinear systems, and none can formulate output-feedback control policies. To the best of our knowledge, the only method that performs active model discrimination for nonlinear systems and has an open-source implementation is \cite{singh_input_2018}. We use it as a baseline and provide runtime comparisons in Sec. \ref{sec:runtimes}.

\subsection{Unicycle Model}
\
We first validate our method on a standard Unicycle system\cite{lavalle2006planning} governed by the following dynamics:
$$    \dot{p}_x = v \cos \phi, 
    \dot{p}_y = v \sin \phi, 
    \dot{v} = u_1 + w_1,
    \dot{\phi} = u_2 + w_2,$$
\noindent where $(p_x, p_y)$ is the car's position, $v$ is its velocity, $\theta$ is its heading angle, $u_1$ and $u_2$ are control inputs, and $w_1, w_2$ are disturbances. For this system, only position $(p_x, p_y)$ can be directly observed.
We consider two fault scenarios:
\begin{itemize}
    \item \textbf{Actuator Fault}: partial to full loss of turning rate control, modeled as $\dot{\phi} = \theta_a u_2 + w_2$ where $\theta_a \in [0, 1)$ represents the degradation factor.
    \item \textbf{Sensor Fault}: a constant bias is applied to position readings, modeled as $y = [p_x, p_y]^T+\theta_o$ where $\theta_o$ represents the sensor bias.
\end{itemize}

Measurements contain random noise with magnitude less than $0.05$m, i.e., $||w||_\infty\le 0.05$.
We compute an open-loop separating controller $\pi(t)$ and a separating output-feedback controller of the form $\pi^*(t, y) = K(t)(y-\hat{y})+u_{ff}(t),$ where $\hat{y}$ is a pre-planned straight-line reference trajectory from start to the obstacle, according to \eqref{opt:active-intersection-refinement} and validate it through hardware experiments, whose results are highlighted in Figure 1(A). Solving for a controller takes 47.2$\text{ms} \pm 766\mu$s.

In the hardware experiments, the robot’s position consistently falls within the refined output reachable sets of the true fault model by the end of the trial, successfully distinguishing between the two fault scenarios while avoiding the obstacle.

\subsection{ADMIRE Fighter Jet Model}
We also validate our proposed method on a simplified version of the ADMIRE Fighter Jet Model \cite{forssell2005admire}. Due to the system's complexity, we refer readers to the original technical report and associated C++ code for the full system dynamics. The simplified model has 9 states and 10 inputs:
\begin{align*}
    \dot{x} &= f(x) + g(x, \theta_a u), \quad
    y = x\\
     x &= 
    \begin{bmatrix}
        V_T & \alpha & \beta & p_b & q_b & r_b & \psi & \theta & \phi
    \end{bmatrix}^T \in \mathbb{R}^9, \\
    f(x)&=
    \begin{bmatrix}
       \dot{ V_T} \\ \dot{\alpha} \\ \dot{\beta}\\ \dot{ p_b}\\ \dot{q_b} \\ \dot{r_b} \\ \dot{\psi} \\ \dot{\theta} \\ \dot{\phi}
    \end{bmatrix}
    =
    \begin{bmatrix}
        (u_b\dot{u}_b + v_b\dot{v}_b+w_b\dot{w}_b)/V_T\\
        (u_b\dot{w}_b - w_b\dot{u}_b)/(u_b^2+w_b^2)\\
        (\dot{v}_bV_T-v_b\dot{V}_T)/(V_T^2\cos(\beta))\\
        (C_1r_b+C_2p_b)q_b\\
        C_5p_br_b-C_6(p_b^2-r_b^2)\\
        (C_8p_b-C_2r_b)q_b\\
        (q_b\sin\phi+r_b\cos\phi)/\cos\theta\\
        q_b\cos\phi-r_b\sin\phi\\
        p_b+\tan\theta(q_b\sin\phi+r_b\cos\phi)\\
    \end{bmatrix}
    \\
\end{align*}  
\vspace{-3em}
\begin{align*}
    u_b &= V_T\cos\alpha\cos\beta,\textbf{ }
    v_b = V_T\sin\beta,\textbf{ }
    w_b=V_T\sin\alpha\cos\beta\\
    \dot{u}_b &= r_bv_b-q_bw_b-g\sin\theta\\
    \dot{v}_b&= -r_bu_b+p_bw_b+g\sin\phi\cos\theta\\
    \dot{w}_b&=q_bu_b-p_bv_b+g\cos\theta\cos\phi\\
    u &\in \mathbb{R}^{10},  \theta_a \in \mathbb{R}^{10 \times 10},     {\theta}_a^* = I_{10}
\end{align*}

The mapping $g(x, \theta_a u)$ is nonlinear and thus modeled in the full system as linear interpolation of values contained in multiple aerodynamics datasheets. While it is certainly possible to implement the full datasheet and immrax is fully capable of producing linear embedding for these otherwise piecewise linear realizations of actuator model, to simplify implementation of the model, we linearize the system's actuator dynamics around the operating point Mach 0.3 and altitude 2000m \cite{bouvier2023resilience}. This yields a locally linear actuator dynamics of the form $\hat{g} = \bar{B}u$. We impose an input constraint $||u||_\infty \leq 0.05$ to ensure linearization validity. Hence, the system as implemented has the following form:
    \begin{equation}
        \dot{x} = f(x) + \bar{B} \text{diag}(\theta_a) u, \quad y = x
    \end{equation}

We define the fault set $\mathcal{M} = \{m_0, m_1, \dots, m_{10}\}$, where $m_0$ is the nominal case and $m_k$ for $k \in \{1, \dots, 10\}$ represents a total loss of control authority in the $k$-th control surface. Specifically, the actuator effectiveness matrix is defined as $\theta_{a,k} = \text{diag}(v)$ where $v_k = 0$ and $v_j = 1$ for $j \neq k$. 
We assume that an output can be sampled from the system and a new input signal applied every $1s$. Running a GPU-parallelized version of the proposed output-anticipating refinement algorithm computes a 5s long sequence of open-loop control inputs that fully separates the 11 fault cases for a runtime of 45.3ms±71.7 $\mu$s on a laptop equipped with Intel i7-13700HX, 32GB RAM, and an RTX4070 with 8GB GPU memory, enabling real-time computation of separating inputs on a physical system. The observed increase in runtime compared to Unicycle is due to both the higher dimensionality of the ADMIRE system model and an increase in number of model pairs considered for discrimination. The mixed integer program baseline method in \cite{singh_input_2018} takes 301s to run on the same device, and fails to find a separating input.

\begin{table}[ht]
\centering
\caption{Output-anticipating refinement (\ref{opt:active-intersection-refinement}) discriminates models faster than uninformed method (\ref{eq:opt_cost}).}
\label{tab:results}
\begin{tabular}{lcccccc}
\toprule
Number of Unseparated Model Pairs & 0s & 1s & 2s & 3s & 4s & 5s \\
\midrule
Uninformed Algorithm (\ref{eq:opt_cost}) & 55 & 17 & 12 & 8 & 5 & 4 \\
Output-Anticipating Refinement (\ref{opt:active-intersection-refinement}) & 55 & 1 & 1 & 1 & 1 & 0 \\
\bottomrule
\end{tabular}
\end{table}

\begin{figure}
    \centering
    \includegraphics[width=\linewidth]{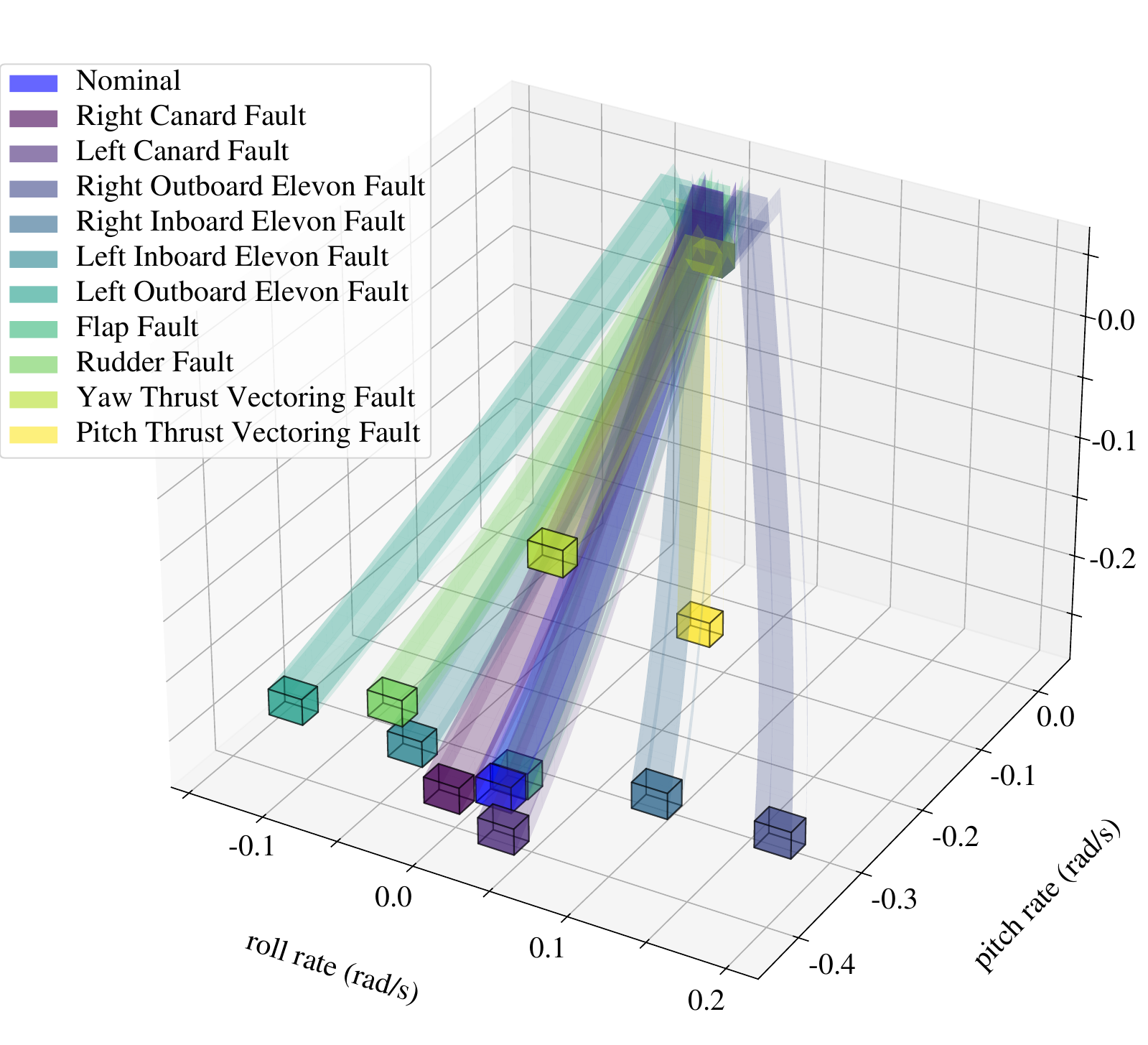}
    \caption{Visualization of the time history of output reachable intervals for the separating controller computed by solving (\ref{opt:active-intersection-refinement}) for the ADMIRE fighter jet system. Our controller fully separates every model pair in 5 seconds.\vspace{-10pt}}
    \label{fig:roll_rate_10}
\end{figure}

\vspace{-8pt}
\subsection{Quadrupedal Navigation}
We extend this framework to model potential actuator and sensor failures in a quadrupedal navigation. The dynamics of a Unitree Go2 quadruped controlled using its high-level SDK can be modeled as a planar rigid body 
with augmented state 
$x=[p_x,p_y,\tilde p_x,\tilde p_y,\phi]^T$ and control $u\in\mathbb{R}^3$, where $p_x,p_y$ denote true horizontal and lateral positions in meters, $v_x,v_y$ denote commanded body-frame velocities in m/s, $\phi$ denotes yaw angle in radians, and $\omega$ is yaw rate in rad/s. Additionally, the states  $\tilde{p}_x, \tilde{p}_y, \tilde{v}_y$ denote odometer measurements of horizontal and lateral positions and velocity, which comes with hardware noise.
Its dynamics can be expressed as:
\begin{align*}
    \dot{p}_x &= v_x \cos\phi - v_y\sin\phi, \quad
\dot{p}_y = v_x\sin\phi + v_y\cos\phi \\
    \dot{\tilde{p}}_x &= v_x \cos\phi - \tilde{v}_y\sin\phi, \quad
\dot{\tilde{p}}_y = v_x\sin\phi + \tilde{v}_y\cos\phi \\
    \tilde{v}_y &= (1-\omega^2)v_y+\omega^2\nu, \quad
    \dot{\phi} = \omega, \quad
        y = \theta_o x\\
    \begin{bmatrix}v_x, v_y, \omega\end{bmatrix} &= (\text{diag}(\theta_a)u)^T\\
    \theta_o^{\text{nominal}} &= [\mathbf{I}_{2\times 2} \text{ } \mathbf{0}_{2\times 2}\text{ }  \mathbf{0}_{2\times 1}], \theta_o^{\text{fault}} = [\mathbf{0}_{2\times 2} \text{ } \mathbf{I}_{2\times 2} \text{ } \mathbf{0}_{2\times 1}]\\
    \theta_a &\in [\underline{\theta_a}, \overline{\theta_a}], \nu \in [\underline{\nu}, \overline{\nu}], {\theta}_a^*=\mathbb{1}_3, \theta_o^*=\mathbb{1}_2
\end{align*}

The system's actuator fault is represented by modifying the actuator outputs: a
loss-of-efficiency fault scales the commanded input by a factor
$\theta_a\in[\underline{\theta_a},\overline{\theta_a}]$,
so that $\theta_{a,i}=0$ corresponds to a stuck actuator and $0<\theta_{a,i}<1$
models partial degradation. 
For sensor fault, this simplified model captures our observation that the quadruped's onboard odometer drifts over time due to inaccuracies in $v_y$ measurements, whose magnitude depends on the magnitude of yaw rate. Hence, we use $ \tilde{v}_y = (1-\omega^2)v_y+\omega^2\nu$ to model this inaccuracy in measured lateral velocities, which is excited only when a turn is commanded to the robot.

To demonstrate our proposed algorithm's ability to handle constraints, we add a collision avoidance constraint to the optimization problem. Assuming there exists a static obstacle at a position $p_o=[-1, 0]$m, we use a circular constraint function $h:=(p_r - p_o)^2 - r_o^2$ and penalize intersection in the optimization. Because the collision avoidance constraint is imposed on $p_x, p_y$, which are not observed in sensor fault, (\ref{eq:opt_cost}) must be used in place of (\ref{opt:active-intersection-refinement}). Solving \eqref{eq:opt_cost} with this collision avoidance constraint in place of \eqref{eq:const_safety} yields a controller with guaranteed obstacle avoidance. This optimization problem can be compiled and solved on a laptop equipped with Intel i7-13700HX, 32GB RAM, and an RTX4070 with 8GB GPU Memory in 40.2 ms ± 25.2 $\mathrm{\mu}$s, demonstrating that our active fault diagnosis algorithm can be solved in real time on a nonlinear system while also respecting safety constraints.

\begin{figure}[h]
\vspace{-0.5em}
    \centering
    \includegraphics[width=\linewidth]{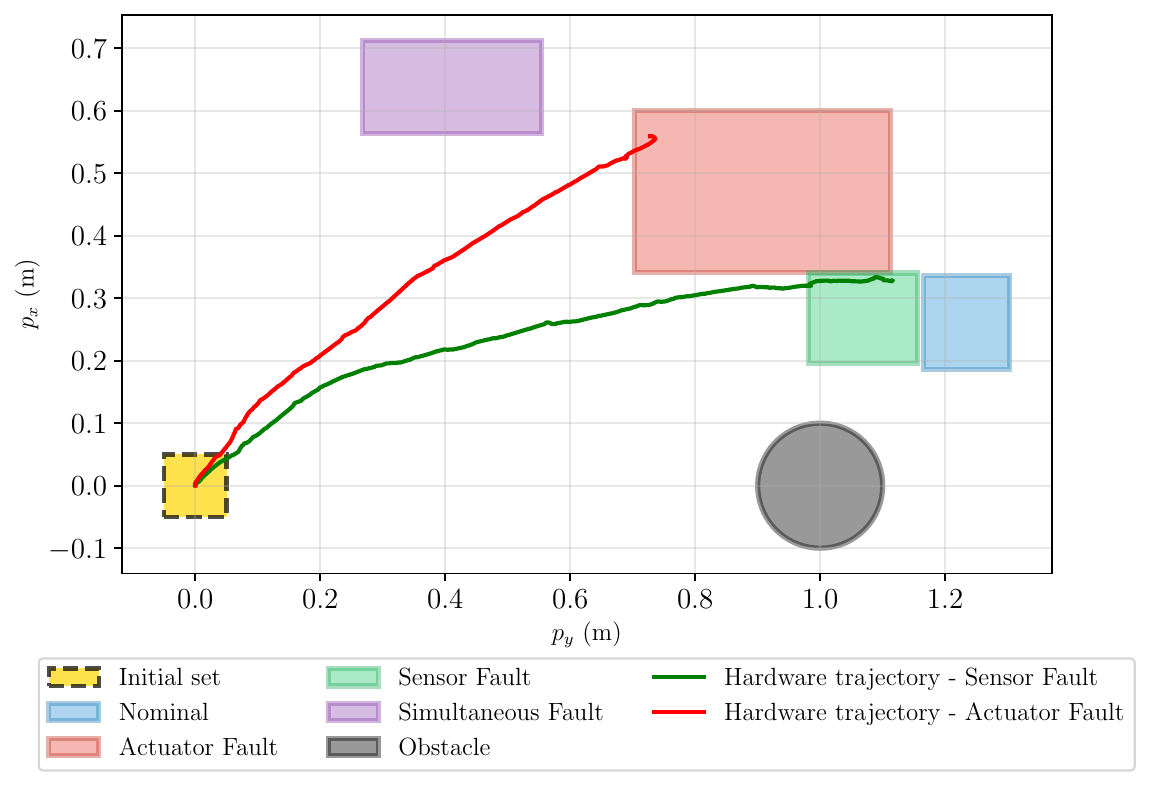}
    \vspace{-1em}
    \caption{Hardware validation of active fault diagnosis on Go2. The solution to (\ref{eq:opt_cost}) results in full separation between output intervals at $T=2s$. Both sensor fault and actuator fault scenarios are implemented on hardware, and the system's final output sets fall within predicted intervals. Notice that the collision avoidance constraint is modelled in state space, not output space.}
    \label{fig:go2_hardware}
    \vspace{-1em}
\end{figure}
\subsection{Runtime Comparison} \label{sec:runtimes}
We compare the runtimes of our implementation against the mixed integer program baseline in \cite{singh_input_2018} across the three robotic systems with different sets of constraints. 
 The runtimes reported for our method use (\ref{opt:active-intersection-refinement}) except for Go2 open loop, where the need to use unmeasured states to satisfy collision avoidance constraint requires the use of (\ref{eq:opt_cost}). 

When setting up the baseline, we use the same system dynamics and set of candidate models for each case, but restrict the control horizon to one step as the baseline can only solve for a single open loop control input. Since the baseline is only capable of solving a single-step open-loop controller, we do not provide a comparison for the case of unicycle output feedback.
\begin{table}[]
\centering
\caption{Runtime comparison between our method and baseline.}
\vspace{-0.5em}
\label{tab:runtimes}
\begin{tabular}{p{3cm}|cc}
\toprule
\centering \textbf{Task/Run Time} & \textbf{Ours} & \textbf{Baseline} \\ 
\hline
Unicycle Open Loop & 3.84ms & solver fails  \\
\hline

Unicycle Output Feedback with Collision Avoidance & 47.2ms & - \\
\hline
ADMIRE Open Loop & 45.3ms & 301s  \\
\hline
Go2 Open Loop & 40.2ms & 22.6s\\
\bottomrule
\end{tabular}
\vspace{-2em}
\end{table}
\vspace{-2em}
\section{Conclusion}
\vspace{0em}
This paper presented a safe, real-time framework for active model discrimination and fault diagnosis in uncertain nonlinear systems. By combining interval reachable-set over-approximations with a differentiable overlap-based objective, the proposed method computes informative control inputs that improve fault separability while explicitly enforcing safety constraints. Across simulation and hardware case studies, the approach demonstrated fast runtimes and reliable discrimination performance for both sensor and actuator fault scenarios.
A key contribution is the output-anticipating refinement strategy (\ref{opt:active-intersection-refinement}), which reduces conservatism by propagating ambiguity-relevant state intervals. A natural next step is to explore this idea by allowing a different controller to be assigned to each model pair, rather than optimizing a shared control law across all pairs.
\bibliographystyle{IEEEtran}
\vspace{-0.5em}
\bibliography{fault_detection}
\end{document}